\documentclass{llncs}
\usepackage[utf8]{inputenc}
\usepackage[english]{babel}
\usepackage{amssymb}
\usepackage{amsmath}
\usepackage{latexsym}
\usepackage{amsfonts}
\usepackage{bbding}
\usepackage{float}
\usepackage{url} 
\usepackage{xspace}
\usepackage{xcolor}
\usepackage{fancybox}
\usepackage{subfigure} 
\usepackage{todonotes} 
\usepackage{pgf}
\usepackage{multirow}
\usepackage{tikz}
\usetikzlibrary{arrows,shapes,automata,backgrounds,petri,matrix,calc,patterns,fit}
\usepackage{array}
\newcolumntype{P}[1]{>{\centering\arraybackslash}p{#1}}

\newcommand{\preset}[1]{{{}^{\bullet}{#1}}}
\newcommand{\postset}[1]{{{#1}^{\bullet}}}

\newcommand{\sys}{\mathcal{S}}

\newcommand\eqdef    {\mathrel{:=}}

\newcommand\bench[2][]{%
\if\relax\detokenize{#1}\relax%
\textsc{#2}\else \textsc{#2}\hspace{0.3pt}{\footnotesize(}#1{\footnotesize)}\fi\@\xspace}

\newcommand{\dist} {\mathit{dist}}
\newcommand{\Language} {\mathcal{L}}
\newcommand{\precision} {\mathit{prec}}

\newcommand{\Art}[2]{\begin{array}[t]{@{}#1@{}}#2\end{array}}

\tikzstyle{transition} = [rectangle, draw, inner sep=0,text width=0.4cm, minimum height=.4cm, text centered]
\tikzstyle{place}=[circle,draw,minimum size=4mm]

\newtheorem{algorithm}{Algorithm}

\newcommand{\sequence}[1]{\ensuremath \langle #1 \rangle}
\begin{document}
\sloppy



\title{Anti-Alignments -- Measuring The Precision of Process Models and Event Logs}
\author{Thomas Chatain \and Mathilde Boltenhagen \and Josep Carmona}
\institute{
LSV, ENS Paris-Saclay, CNRS, Inria, Université Paris-Saclay, Cachan (France)
\\
\email{chatain@lsv.fr}
\\
LSV, CNRS, ENS Paris-Saclay, Inria, Université Paris-Saclay, Cachan (France)
\\
\email{boltenhagen@lsv.fr}
\\
Universitat Polit\`ecnica de Catalunya, Barcelona (Spain)
\\
\email{jcarmona@cs.upc.edu}
}

\maketitle

\begin{abstract}
Processes are a crucial artefact in organizations, since they coordinate
the execution of activities so that products and services are provided.
The use of models to analyse the underlying processes is a well-known practice. 
However, due to the complexity and continuous evolution of their processes, 
organizations need an effective way of analysing the relation between processes and models.
Conformance checking techniques asses the suitability of a process model in representing 
an underlying process, observed through a collection of real executions.
%
One important metric in conformance checking is to asses the precision of the model with
respect to the observed executions, i.e., characterize the ability of the model
to produce behavior unrelated to the one observed.
In this paper we present the notion of {\em anti-alignment} as a concept to help unveiling runs in the model
that may deviate significantly from the observed behavior. Using anti-alignments, a new metric for precision is
proposed. In contrast to existing metrics, anti-alignment based precision metrics satisfy most of the required 
axioms highlighted in a recent publication.  
Moreover, a complexity analysis of the problem of computing anti-alignments is provided, which sheds light into
the practicability of using anti-alignment to estimate precision.
Experiments are provided that witness
the validity of the concepts introduced in this paper.
\end{abstract}





\section{Introduction}



Relating observed and modelled process behavior is the lion's share of conformance checking~\cite{CarmonaDSW18}. Observed behavior is
often recorded in form of {\em event logs}, that store the footprints of process executions. Symmetrically, process models are 
representations of the underlying process, which can be automatically discovered or manually designed.
With the aim of quantifying this relation, conformance checking techniques consider four quality dimensions: fitness, precision,
generalization and simplicity~\cite{RozinatA08}. For the first three dimensions, the {\em alignment} between a process model and 
an event log is of paramount importance, since it allows relating modeled and observed behavior~\cite{AryaThesis}. 

Given a process model and a trace in the event log, an alignment provides  the run in the model which mostly resembles the observed trace. When alignments 
are computed, the quality dimensions can be defined on top~\cite{AryaThesis,JorgeMunozPhD}. In a way, alignments are optimistic: although observed behavior 
may deviate significantly from modeled behavior, it is always assumed that the least deviations are the best explanation (from the model's perspective) for 
the observed behavior.

In this paper we present a somewhat symmetric notion to alignments, denoted as {\em anti-alignments}. Given a process model and a log, an anti-alignment is a run
of the model that mostly deviates from any of the traces observed in the log. The motivation for anti-alignments is precisely to compensate the optimistic view
provided by alignments, so that the model is queried to return highly deviating behavior that has not been seen in the log. In contexts where the process model
should adhere to a certain behavior and not leave much room for exotic possibilities (e.g., banking, healthcare), the absence of highly deviating anti-alignments may 
be a desired property for a process model. Using anti-alignments one cannot only catch deviating behavior, but also use it to improve some of the current quality metrics considered in conformance checking. In this paper we highlight the strong relation of anti-alignments and the precision metric: a highly-deviating anti-alignment may be considered as a witness for a loss in precision. Current metrics for precision lack this ability of exploring the model behavior beyond 
what is observed in the log, thus being considered as short-sighted~\cite{AdriansyahMCDA15}.

We cast the problem of computing anti-alignments as the satisfiability of a Boolean formula, and provide high-level techniques which can for instance compute the 
most deviating anti-alignment for a certain run length, or the shortest anti-alignment for a given number of deviations.

Anti-alignments are related to the {\em completeness of the log}; a log is complete if it contains all the behavior of the underlying 
process~\cite{AalstBook}. For incomplete logs, the alternatives for computing anti-alignments grow, making it difficult to tell the difference between
behavior not observed but meant to be part of the process, and behavior not observed which is not meant to be part of the process. Since there exists
already some metrics to evaluate the completeness of an event log (e.g.,~\cite{Yang12}), we assume event logs have a high level of completeness
before they are used for computing anti-alignments. Notice that in presence of an incomplete event log, anti-alignments can be used to interactively complete it: an anti-alignment that is certified by the stake-holder as valid process behavior can be appended to the event log to make it more complete. 

This work is an extension of recent publications related to anti-alignments: in~\cite{DBLP:conf/apn/ChatainC16} we established 
for the first time the notion of anti-alignments based on the Hamming distance, and proposed a simple metric to estimate precision. Then, the work in~\cite{DBLP:conf/bpm/DongenCC16} elaborated the notion of anti-alignments, heuristically computing them for the Levenshtein distance by adapting the $A^*$ search technique, and proposed two new notions for trace-based and log-based precision, that can be combined to estimate precision of process models. However, as it was claimed recently in a survey paper advocating for properties precision
metrics should have~\cite{TAX20181}, it was not known the satisfiability of the properties for the aforementioned metrics.

The contributions of the paper with respect to our previous work are now enumerated.

\begin{itemize}
\item We show how anti-alignments can be computed in an optimal way for the Levenshtein distance, without increasing the complexity class of the problem. Moreover we relate the two available distance encodings (Hamming and Levenshtein), and show the implications of using each one for anti-alignment based precision.
 \item We adapt the precision metrics from~\cite{DBLP:conf/bpm/DongenCC16} to not depend on 
 a particular length defined apriori.
 \item We prove the adherence of one of the new metrics proposed in this paper to most of the properties in~\cite{TAX20181}.
\item A novel implementation is provided, with several improvements, which makes it able to deal with larger instances.
\item A new evaluation section is reported, that show empirically the capabilities of the proposed technique for large and real-life instances.
\end{itemize}

The remainder of the paper is organized as follows: in the next section, a
simple example is used to emphasize the importance of anti-alignments and its application to
estimate precision is shown.
Then in Section~\ref{sec:preliminaries} the basic theory needed for the
understanding of the paper is introduced. Section~\ref{sec:anti-alignments}
provides the formal definition of anti-alignments, whilst
Section~\ref{sec:encoding} formalizes the encoding into SAT of the problem of
computing anti-alignments.
In Section~\ref{sec:precision}, we define a new metric, based on
anti-alignments, for estimating precision of process models.
Experiments are reported in Section~\ref{sec:experiments}, and related work in
Section~\ref{sec:related}.
Section~\ref{sec:conclusions} concludes the paper and gives some hints for
future research directions.

\section{A Motivating Example}


Let us use the example shown in Figure~\ref{fig:running} for illustrating the notion of anti-alignment. The example was originally
presented in~\cite{BrouckeMCBV14}, and in this paper we present a very abstract version of it in Figure~\ref{fig:running}(a): 
The modeled process describes a realistic transaction process within a banking context. The 
process contains all sort of monetary checks, authority notifications, and logging mechanisms. 
The process is initiated when 
a new transaction is requested, opening a new instance in the system and registering all the components involved. The second step 
is to run a check on the person (or entity) origin of the monetary transaction. Then, the actual payment is processed differently, 
depending of the payment modality chosen by the sender (cash, cheque and payment). Later, the receiver is checked and the money is 
transferred. Finally, the process ends registering the information, notifying it to the required actors and authorities, and 
emitting the corresponding receipt. 

\begin{figure}[t]
\centering
\includegraphics[width=10.5cm]{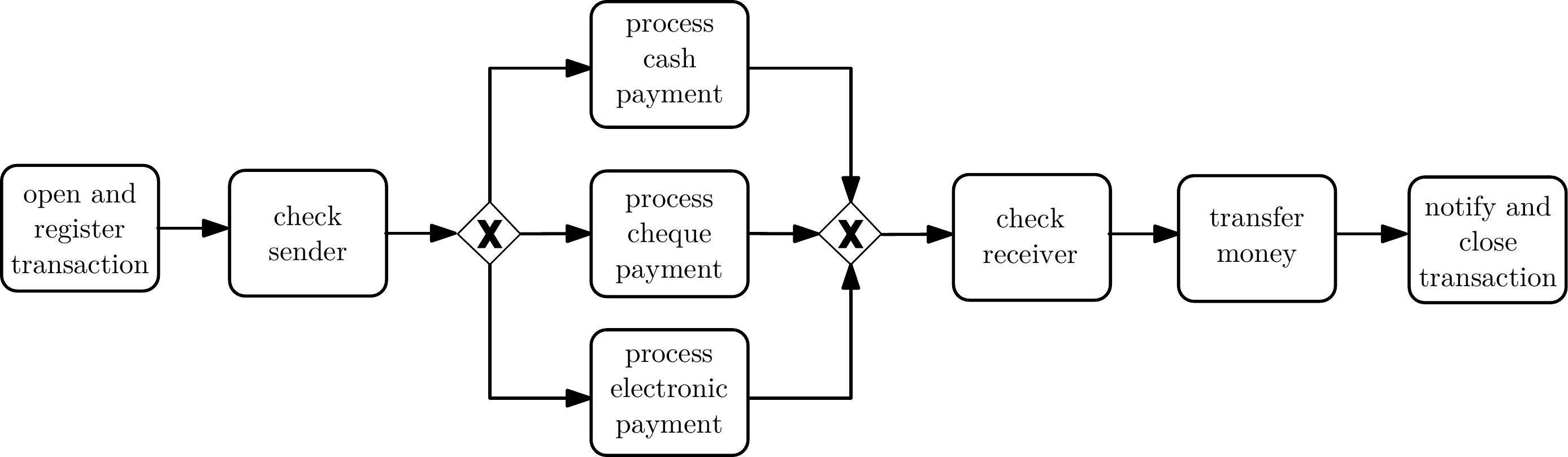}\\
(a)\\
\vspace*{0.25cm}
\includegraphics[width=10.5cm]{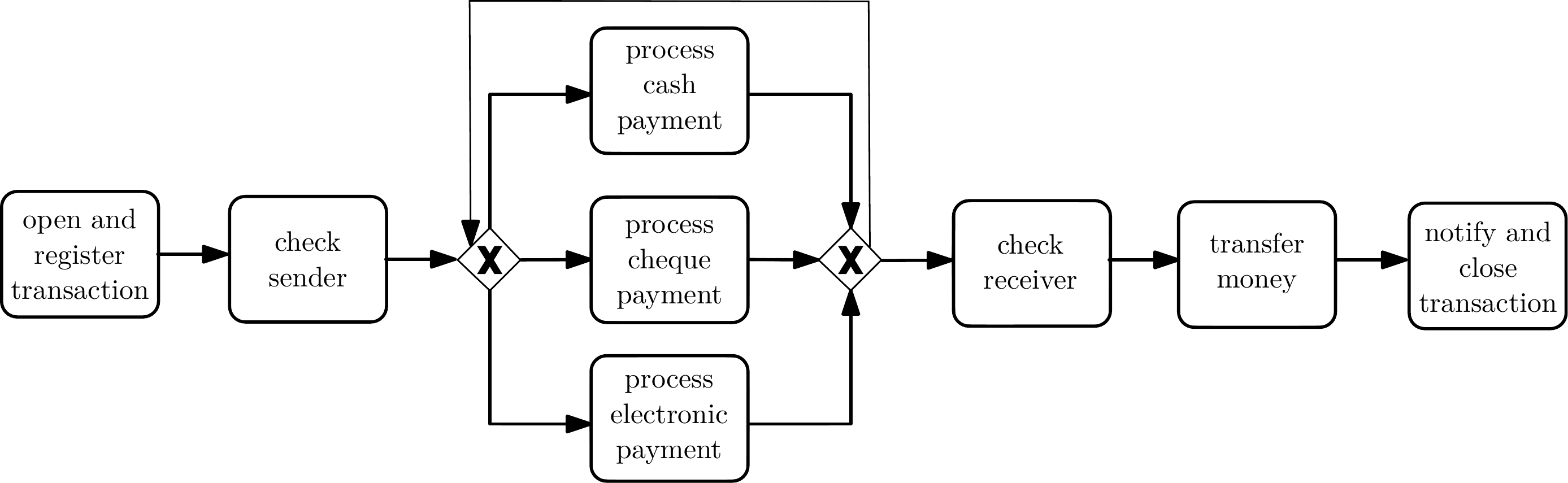}\\
(b)
\caption{Example (adapted from~\cite{BrouckeMCBV14}). (a) Initial process model, (b) Modified process model.}
\label{fig:running}
\end{figure}

Assume that a log covering all the three possible variants (corresponding to the three possible payment methods) with respect of the model in Figure~\ref{fig:running}(a) is given. The three different variants for this log will be:
\begin{center}
{\tt ort, cs, pcap, cr, tm, nct}\\
{\tt ort, cs, pchp, cr, tm, nct}\\
{\tt ort, cs, pep, cr, tm, nct} $\ $\\
\end{center}
\noindent where we use the acronym for each one of the actions performed, e.g., {\tt ort} stands for {\em open and register transaction}.

For this pair of model and log, most of the current metrics for precision (e.g.,~\cite{AdriansyahMCDA15}) will rightly assess a very high precision. In fact, since no deviating anti-alignment can be obtained because every model run is in the log, the anti-alignment based precision metric from 
this paper will also assess a high (in our case, perfect) precision.

Now assume that we modify a bit the model, adding a loop around the alternative stages for the payment. Intuitively, 
this (malicious) modification in the process model may allow to pay several times although only one transfer will be done. The modified high-level
overview is shown in Figure~\ref{fig:running}(b). The aforementioned metric for precision will not consider this 
modification as a severe one: the precision of the model with respect to the log will be very similar to the one for the model in Figure~\ref{fig:running}(a).


Remarkably, this modification in the process model comes with a new highly
deviating anti-alignment denoting a run of the model that contains more than one
iteration of the payment:
\begin{center}
{\tt ort, cs, pcap, pchp, pchp, pep, pcap, cr, tm, nct}\\
\end{center}
Clearly,
this model execution where five payments have been recorded is possible in the process of Figure~\ref{fig:running}(b). Correspondingly, the precision
of this model in describing the log of only three variants will be significantly lowered in the metric proposed in this paper, since the anti-alignment
produced is very different from any of the three variants recorded in the event log.


\section{Preliminaries}
\label{sec:preliminaries}


\newcommand{\arc}[2]{#1\rightarrow #2}
\newcommand{\larc}[3]{#1\stackrel{#2}{\rightarrow}#3}
\newcommand{\longarc}[2]{#1\longrightarrow #2}
\newcommand{\longlarc}[3]{#1\stackrel{#2}{\longrightarrow}#3}
\newcommand{\firing}[1]{[#1\rangle}        



\begin{definition}[(Labeled) Petri net]
  A (labeled) Petri Net~\cite{Murata89} is a tuple $N = \langle P,T,\mathcal{F}, M_{0}, M_f, \Sigma, \lambda \rangle$, where
  $P$ is the set of places, $T$ is the set of transitions (with $P \cap T = \emptyset$), $\mathcal{F}: (P \times T) \cup (T \times P) \to \{0,1\}$ is the flow relation, $M_0$ is the initial marking, $M_f$ is the final marking,
  \(\Sigma\) is an alphabet of actions and \(\lambda : T \to \Sigma\) labels every transition by an action.

%
\end{definition}

A marking
is an assignment of a non-negative integer to each place. If $k$ is assigned to place $p$
by marking $M$ (denoted $M(p) = k$), we say that $p$ is marked with $k$ tokens. Given a node 
$x \in P \cup T$, its pre-set and post-set are denoted by $\preset{x}$ and $\postset{x}$
respectively. 

A transition $t$ is {\em enabled} in a marking $M$ when all places in
$\preset{t}$ are marked. When a transition $t$ is enabled, it can {\em fire\/} by removing a token
from each place in $\preset{t}$ and putting a token to each place in $\postset{t}$.
A marking $M'$ is {\em reachable} from $M$ if there is a sequence of firings
$t_1 \ldots t_n$ that transforms $M$ into $M'$, denoted by ${M}\firing{t_1 \ldots t_n}{M'}$.
We define the \emph{language} of \(N\) as the set of \emph{full runs}
defined by \(\Language(N) \eqdef \{\lambda(t_1) \ldots \lambda(t_n) \mid
{M_0}\firing{t_1 \ldots t_n}{M_f}\}\). 
A Petri net is
{\em k-bounded} if no reachable marking assigns more than $k$ tokens to
any place. A Petri net is bounded if there exist a $k$ for which it is $k$-bounded. A Petri net is {\em safe} if it is 1-bounded. 
A bounded Petri net has an {\em executable loop} if it has a 
reachable marking \(M\) and sequences of transitions  \(u_1 \in T^*\), \(u_2 \in T^+\), \(u_3 \in T^*\) such that
  \({M_0}\firing{u_1}M\firing{u_2}M\firing{u_3}{M_f}\).


\medskip

An event log is a collection of traces, where a trace may appear more than once. Formally:
\begin{definition}[Event Log]
An \emph{event log} $L$ (over an alphabet of actions \(\Sigma\)) is a multiset
of traces \(\sigma \in \Sigma^*\).
\end{definition}

%



Process mining techniques aim
at extracting from a log $L$ a process model $N$ (e.g., a Petri net) with the goal to elicit the process
underlying a system ${\sys}$. ${\sys}$ is considered a language for the sake of comparison. By relating the behaviors of $L$, $\Language(N)$ and ${\sys}$,
particular concepts can be defined~\cite{BuijsDA14}. A log is \emph{incomplete} if ${\sys} \setminus L \ne
\emptyset$. 
A model $N$ \emph{fits} log $L$ if $L \subseteq \Language(N)$. A model is
\emph{precise} in describing a log $L$ if $\Language(N) \setminus L$ is small. A model $N$
represents a \emph{generalization} of log $L$ with respect to system ${\sys}$ if some behavior in ${\sys}
\setminus L$ exists in $\Language(N)$. Finally, a model $N$ is \emph{simple} when it has the minimal
complexity in representing $\Language(N)$, i.e., the well-known \emph{Occam's razor principle}.


\section{Anti-Alignments}
\label{sec:anti-alignments}

The idea of \emph{anti-alignments} is to seek in the language of a model \(N\) what
are the runs which differ considerably with all the observed traces.
Hence, this is the opposite of the notion of \emph{alignments}~\cite{AryaThesis}
which is central in process mining: for many tasks in conformance
checking like {\em process model repair} or {\em decision point analysis}, one needs indeed to find the run which is the most
similar to a given log trace~\cite{AalstBook}.
In this paper, we are focusing on precision and for this, traces which are not
similar to any observed trace in the log serve as witnesses for bad precision.
All these notions anyway depend on a definition of distance between two traces
(typically a model trace, i.e.\ a run of the model, and an observed log trace).
We assume a given distance function \(\dist: \Sigma^* \times \Sigma^* \rightarrow [0,1]\)
computable in polynomial time and such that
\footnote{Actually, we do not require that \(\dist\) satisfies the usual
  properties of distance functions like symmetry or triangle inequality.}
\begin{itemize}
\item for every \(\gamma \in \Sigma^*\), \(\dist(\gamma, \gamma) = 0\),
\item for every \(\sigma \in \Sigma^*\), \(\dist(\gamma, \sigma)\) converges to
  \(1\) when \(|\gamma|\) diverges to \(\infty\).
\end{itemize}

\begin{definition}[Anti-alignment]\label{def-antialignment}
  For a distance threshold \(m \in [0,1]\),
  an \(m\)-\emph{anti-alignment} of a model \(N\) w.r.t.\ a log \(L\) is a
  full run \(\gamma \in \Language(N)\) such that \(\dist(\gamma, L) \geq m\),
  where \(\dist(\gamma, L)\) is defined as the \(\min_{\sigma \in
    L}\dist(\gamma, \sigma)\).\footnote{Since the function \(\dist\) takes its
    values in \([0, 1]\), we define by convention \(\dist(\gamma, \emptyset) \eqdef 1\).}

\end{definition}

For the following examples, we show anti-alignments w.r.t.\ two possible
choices of distance \(\dist\): Levenshtein's distance and Hamming distance.

\begin{definition}[Levenshtein's edit distance \boldmath\(\dist^L\)]
  \label{def:Levenshtein}
  Levenshtein's edit distance \(\dist^L(\gamma, \sigma)\) between two
  traces \(\gamma, \sigma \in \Sigma^*\) is based on the minimum number \(n\) of
  deletions and insertions needed to transform \(\gamma\) to \(\sigma\).
  In order to get a normalized distance between 0 and 1, we define Levenshtein's
  edit distance \(\dist^L(\gamma, \sigma) = \frac{n}{\max(1, |\gamma| + |\sigma|)}\).
\end{definition}

\begin{example}
\begin{figure}[t]
\centering
\includegraphics[width=10.5cm]{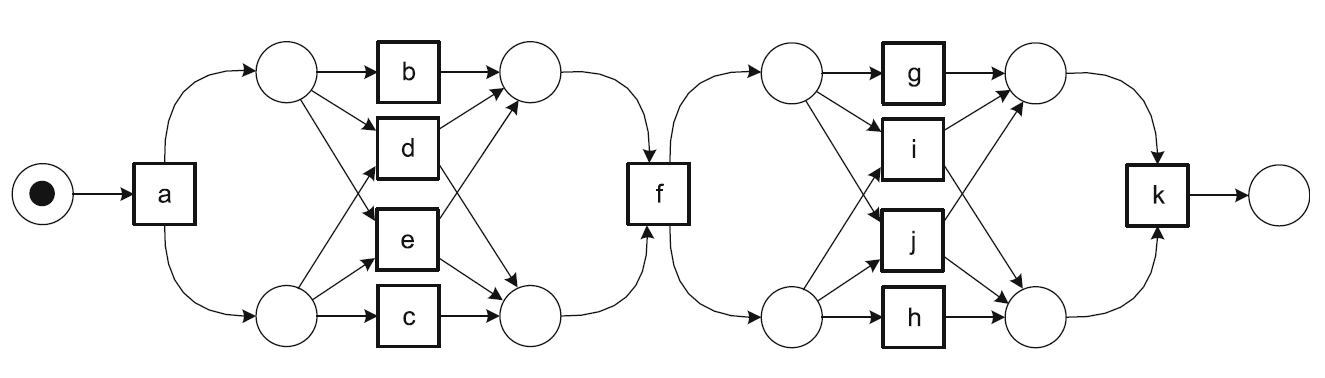}
\[L = \{\Art{l}{
  \sequence{a,b,c,f,g,h,k},\\
  \sequence{a,c,b,f,g,h,k},\\
  \sequence{a,c,b,f,h,g,k},\\
  \sequence{a,b,c,f,h,g,k},\\
  \sequence{a,e,f,i,k},\\
  \sequence{a,d,f,g,h,k},\\
  \sequence{a,e,f,h,g,k}\}\,,}\]
\caption{A process model (taken from~\cite{AalstKRV15}) and an event log. The
  full run $\sequence{a,b,c,f,i,k}$ is a \(\frac{3}{13}\)-anti-alignment for Levenshtein's
  distance and a \(\frac{3}{7}\)-anti-alignment for Hamming distance.
}
\label{fig:anti-alignment-ex}
\end{figure}
Consider the Petri net and log shown in Figure~\ref{fig:anti-alignment-ex}. With
Levenshtein's distance, the full run $\sequence{a,b,c,f,i,k}$ is at distance \(\frac{3}{13}\)
from the log trace $\sequence{a,b,c,f,g,h,k}$ (two deletions and one insertion).
It is at larger distance from the other log traces. Therefore, it is a
\(\frac{3}{13}\)-anti-alignment.
\end{example}

Another interesting choice is Hamming distance. It is in general less 
informative than Levenshtein's distance 
for relating observed and modelled behavior,
but it has the
interest of being very simple to compute. Variants of Hamming distance can also
provide good compromises. In Sections~\ref{sec:encoding} and
\ref{sec:precision}, we will show how to efficiently compute anti-alignments for
Hamming distance using SAT solvers.

\begin{definition}[Hamming distance \boldmath\(\dist^H\)]
  \label{def:Hamming_dist}
  For two traces \(\gamma = \gamma_1 \dots \gamma_n\) and \(\sigma = \sigma_1
  \dots \sigma_n\), of same length \(n > 0\), define \(\dist^H(\gamma, \sigma)
  \eqdef \frac1n \cdot \big|\big\{i \in \{1 \dots n\} \mid \gamma_i \neq \sigma_i\big\}\big|\).
  For \(\gamma\) longer than \(\sigma\), we define
  \(\dist^H(\gamma, \sigma) \eqdef \dist^H(\gamma, \sigma \cdot w^{|\gamma| -
    |\sigma|})\), where \(w \not\in \Sigma\) is a special padding
  symbol (\(w\) for `wait'); we proceed symmetrically when \(\gamma\) is shorter than \(\sigma\).
\end{definition}

\begin{lemma}
  \label{lem:pH_lt_pL}
  Observe that, for every \(\gamma\) and \(\sigma\) (assuming one of them at
  least is nonempty), \(\dist^L(\gamma, \sigma) \leq \dist^H(\gamma, \sigma)\).
\end{lemma}
\begin{proof}
  Assume w.l.o.g.\ \(0 < |\gamma| < |\sigma|\).
  Let \(n \eqdef \big|\big\{i \in \{1 \dots |\gamma|\} \mid \gamma_i \neq \sigma_i\big\}\big|\).
  We have, \(\dist_H(\gamma, \sigma) = \frac{n+|\sigma|-|\gamma|}{|\sigma|}\).
  We have also \(\dist_L(\gamma, \sigma) \leq
  \frac{2n+|\sigma|-|\gamma|}{|\gamma|+|\sigma|}\) because one way to transform
  \(\gamma\) to \(\sigma\) is to replace \(\gamma_i\) by \(\sigma_i\) (one
  deletion and one insertion) at each position \(i \leq |\gamma|\) where they
  differ (\(2n\) editions), and then to insert the letters \(\sigma_{|\gamma|+1}
  \dots \sigma_{|\sigma|}\) (\(|\sigma|-|\gamma|\) editions).
  It remains to see that \(n \leq |\gamma|\), which implies
  \begin{eqnarray*}
    n(|\sigma|-|\gamma|) & \leq & |\gamma|(|\sigma|-|\gamma|)\\
    n|\sigma| & \leq & |\gamma|(n+|\sigma|-|\gamma|)\\
    n|\sigma|+(n+|\sigma|-|\gamma|)|\sigma| & \leq & |\gamma|(n+|\sigma|-|\gamma|)+(n+|\sigma|-|\gamma|)|\sigma|\\
    (2n+|\sigma|-|\gamma|)|\sigma| & \leq & (|\gamma|+|\sigma|)(n+|\sigma|-|\gamma|)\\
    \frac{2n+|\sigma|-|\gamma|}{|\gamma|+|\sigma|} & \leq & \frac{(n+|\sigma|-|\gamma|)}{|\sigma|}\\
    \dist^L(\gamma, \sigma) & \leq & \dist^H(\gamma, \sigma)\;.
  \end{eqnarray*} \qed
\end{proof}

\begin{example}
Consider the Petri net and log shown in Figure~\ref{fig:anti-alignment-ex}. With
Hamming distance, the full run $\gamma = \sequence{a,b,c,f,i,k}$ is at distance \(\frac{3}{7}\)
from the log trace $\sigma = \sequence{a,b,c,f,g,h,k}$ ($i$ and $k$ do not match with $g$
and $h$, and \(\gamma\) is shorter than \(\sigma\), which counts for the third mismatch).
It is at larger distance from the other log traces. Therefore, it is a
\(\frac{3}{7}\)-anti-alignment.
\end{example}

\begin{lemma}\label{lemma-largest_m}
  For every log \(L\) and finite model \(N\), we have:
  \begin{enumerate}
  \item if the model has finitely many full runs, then there exists (at least)
    one maximal anti-alignment \(\gamma\) of \(N\) w.r.t.\ \(L\),
    i.e.\ \(\gamma\) that maximizes the distance \(\dist(\gamma, L)\);
  \item if the model has infinitely many full runs, then there exist
    anti-alignments \(\gamma\) with \(\dist(\gamma, L)\) arbitrarily close
    to \(1\). Yet, there may not exist any $1$-anti-alignment, i.e.\ there
    is no guarantee that the limit \(\dist(\gamma, L) = 1\) is reached for
    any \(\gamma\).
  \end{enumerate}
\end{lemma}
\begin{proof}
  If the model has finitely many full runs, then one of them must be a maximal
  anti-alignment.

  Conversely, if \(N\) is finite and has infinitely many runs, then there must
  exist arbitrary long full runs; more formally, there exists an infinite
  sequence of full runs \((\gamma_i)_{i \in \mathbb{N}}\) of strictly increasing
  length. For every \(\sigma \in L\), the sequence of
  \((\dist(\gamma_i, \sigma))_{i \in \mathbb{N}}\) converges to \(1\). Since
  there are finitely many \(\sigma \in L\),
  \((\min_{\sigma \in L}\dist(\gamma_i, \sigma))_{i \in \mathbb{N}}\)
  converges to \(1\) as well.
  \qed
\end{proof}
Maximal anti-alignments will be used in Section~\ref{sec:precision} to define
our precision metric. The case of models with executable loops will be discussed
in Subsection~\ref{sec:loops}.

\begin{lemma}\label{lemma-reachability}
  The problem of deciding, given a finite model \(N\) and a log \(L\),
  whether there exists a \(0\)-anti-alignment of \(N\) w.r.t.\ \(L\), has the
  same complexity as reachability for Petri nets.
\end{lemma}
\begin{proof}
  This is equivalent to checking whether \(\Language(N) \neq \emptyset\), i.e.
  whether \(M_f\) is reachable from \(M_0\).\qed
\end{proof}

By definition, a \(0\)-anti-alignment of \(N\) w.r.t.\ \(L\) is a full run
\(\gamma \in \Language(N)\) satisfying the trivial inequality \(\dist(\gamma, L)
\geq 0\). The same problem with a strict inequality is also
of interest. We will need it in Section~\ref{sec:loops}.

\begin{lemma}\label{lemma-precision1}
  The problem of deciding, given a finite model \(N\) and a log \(L\),
  whether there exists a full run \(\gamma \in \Language(N)\) satisfying
  \(\dist(\gamma, L) > 0\) (or equivalently deciding if \(\Language(N)
  \not\subseteq L\)), has the same complexity as
  reachability for Petri nets.
\end{lemma}
\begin{proof}
  The reachability problem reduces to the existence of a full run \(\gamma \in
  \Language(N)\) satisfying \(\dist(\gamma, \emptyset) > 0\): indeed
  \(\dist(\gamma, \emptyset) = 1\) for every \(\gamma\).

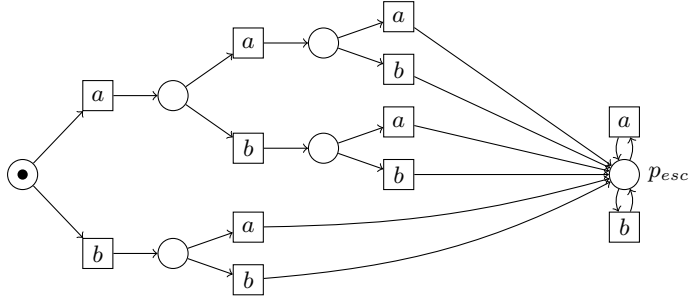
\begin{figure}[tb]
  \begin{center}
    \begin{tikzpicture}[yscale=0.7]
      \node[place,tokens=1] (p) at (0, 1.5) {};
      \node[place] (pa) at (2, 3) {};
      \node[place] (paa) at (4, 4) {};
      \node[place] (pab) at (4, 2) {};
      \node[place] (pb) at (2, 0) {};
      \node[place] (pesc) at (8, 1.5) [label=right:\(p_{esc}\)] {};

      \node[transition] (ta) at (1, 3) {$a$};
      \node[transition] (tb) at (1, 0) {$b$};
      \node[transition] (taa) at (3, 4) {$a$};
      \node[transition] (tab) at (3, 2) {$b$};
      \node[transition] (tba) at (3, 0.5) {$a$};
      \node[transition] (tbb) at (3, -0.5) {$b$};
      \node[transition] (taaa) at (5, 4.5) {$a$};
      \node[transition] (taab) at (5, 3.5) {$b$};
      \node[transition] (taba) at (5, 2.5) {$a$};
      \node[transition] (tabb) at (5, 1.5) {$b$};

      \path[->] (p) edge (ta) edge (tb);
      \path[->] (ta) edge (pa);
      \path[->] (tb) edge (pb);

      \path[->] (pa) edge (taa) edge (tab);
      \path[->] (taa) edge (paa);
      \path[->] (tab) edge (pab);

      \path[->] (pb) edge (tba) edge (tbb);
      \path[->] (tba) edge [bend right=10] (pesc);
      \path[->] (tbb) edge [bend right=15] (pesc);

      \path[->] (paa) edge (taaa) edge (taab);
      \path[->] (pab) edge (taba) edge (tabb);

      \path[->] (taaa) edge [bend right=0] (pesc);
      \path[->] (taab) edge [bend right=0] (pesc);
      \path[->] (taba) edge [bend right=0] (pesc);
      \path[->] (tabb) edge [bend right=0] (pesc);

      \node[transition] (tea) at (8, 2.5) {$a$};
      \path[->] (pesc) edge [bend right=15] (tea);
      \path[->] (tea) edge [bend right=15] (pesc);

      \node[transition] (teb) at (8, 0.5) {$b$};
      \path[->] (pesc) edge [bend right=15] (teb);
      \path[->] (teb) edge [bend right=15] (pesc);

    \end{tikzpicture}
  \end{center}
  \caption{A deterministic Petri net representing the log \(L = \{\sequence{a},
    \sequence{a,b}, \sequence{a,a}, \sequence{b}\}\), for the alphabet of
    actions \(\Sigma = \{a, b\}\). Place \(p_{esc}\) is reached only by the runs
    which do not appear in the log.}
  \label{fig:tree-L}
\end{figure}
  Conversely, deciding if \(\Language(N) \not\subseteq L\) reduces to deciding
  reachability of \(M_f \cup \{p_{esc}\}\) in the synchronous product of \(N\)
  with a deterministic Petri net which represents as a tree the log traces
  sharing their common prefixes, and, from the leaves, marks a sink place
  \(p_{esc}\), as illustrated in Figure~\ref{fig:tree-L}. Hence, every full run
  \(\gamma\) of \(N\), when synchronized with the Petri net representation of
  the log \(L\), leads to a marking of the form \(M_f \cup \{p\}\), and \(p =
  p_{esc}\) iff \(\gamma \not\in L\).
  \qed
\end{proof}

The problem of reachability in Petri nets is known to be decidable, but
non-elementary~\cite{DBLP:conf/stoc/CzerwinskiLLLM19}.

Yet, the complexity drops to NP if a bound is given on the length of the
anti-alignment.
\begin{lemma}\label{lemma-complexity_anti-alignment-bounded-length}
  The problem of deciding, for a Petri net \(N\), a log \(L\), a rational
  distance threshold \(m \in \mathbb{Q}\) and a bound \(n \in \mathbb{N}\), if
  there exists a \(m\)-anti-alignment \(\gamma\) such that \(|\gamma| \leq n\),
  is NP-complete.
  We assume that \(n\) is encoded in unary.\footnote{
    Since \(n\) has typically the same order of magnitude as the length of the
    longest traces in the log, encoding \(n\) in unary does not significantly
    affect the size of the problem instances.}
\end{lemma}
\begin{proof}
  The problem is clearly in NP: checking that a run \(\gamma\) is a
  \(m\)-anti-alignment of \(N\) w.r.t.\ \(L\) takes polynomial time (remember
  that we consider distance functions computable in polynomial time).

  For NP-hardness, we propose a reduction from the problem of reachability of a
  marking \(M_f\) in a 1-safe acyclic\footnote{a Petri net is acyclic if the
    transitive closure \(\mathcal{F}^+\) of its flow relation is irreflexive.}
  Petri net \(N\), known to be NP-complete
  \cite{DBLP:journals/fuin/Stewart95,DBLP:journals/tcs/ChengEP95}.
  Notice that, since \(N\) is acyclic, each transition can fire only once;
  hence, the length of the firing sequences of \(N\) is bounded by the number of
  transitions \(|T|\). Finally, \(M_f\) is reachable in \(N\) iff there exists a
  \(\gamma\) of length less or equal to \(|T|\) which is a \(0\)-anti-alignment
  of \(N\) (with \(M_f\) as final marking) w.r.t.\ the empty log.
  \qed
\end{proof}

\section{SAT-encoding of Anti-Alignments}
\label{sec:encoding}

In this section, we give hints on how SAT solvers can help to find anti-alignments.
We detail the construction of a SAT formula \(\Phi^n_m(N,L)\), where \(N\) is a
Petri net, \(L\) a log, \(n\) and \(m\) two integers. This formula will be used
in the search of anti-alignments of \(N\) w.r.t.\ \(L\) for Hamming distance (see Section~\ref{sec:sat_edit_dist} for 
the encoding using the Levenshtein distance).
The formula \(\Phi^n_m(N,L)\) characterizes precisely the full runs
\(\gamma = \lambda(t_1) \dots \lambda(t_n) \in \Language(N)\)
of \(N\) of length \(n\) which differs in at least \(m\) positions
with every log trace in \(L\).

\subsection{Coding \boldmath\(\Phi^n_m(N,L)\) Using Boolean Variables}
\label{sec-SAT}
The formula \(\Phi^n_m(N,L)\) is coded using the following Boolean variables:
\begin{itemize}
\item \(\tau_{i, t}\) for \(i = 1 \dots n\), \(t \in T\)
  (remind that \(w\) is the special symbol used to pad the log traces, see
  Definition~\ref{def:Hamming_dist})
  means that transition \(t_i = t\).
\item \(m_{i, p}\) for \(i = 0 \dots n\), \(p \in P\) means that place \(p\)
  is marked in marking \(M_i\) (remind that we consider only safe nets,
  therefore the \(m_{i, p}\) are Boolean variables).
\item \(\delta_{i, \sigma, k}\) for \(i = 1 \dots n\), \(\sigma \in L\), \(k = 1, \dots,
  m\) means that the \(k\)\textsuperscript{th} mismatch with the observed trace
  \(\sigma\) is at position \(i\).
\end{itemize}
The total number of variables is \(O(n \times (|T| + |P| + |L| \times m))\).
\medskip

Let us decompose the formula \(\Phi^n_m(N,L)\).
\begin{itemize}
\item The fact that \(\gamma = \lambda(t_1) \dots \lambda(t_n) \in
  \Language(N)\) is coded by the conjunction of the following formulas:
  \begin{itemize}
  \item Initial marking:
    \[
    \left(\bigwedge_{p \in M_0} m_{0, p}\right) \land
    \left(\bigwedge_{p \in P \setminus M_0} \lnot m_{0, p}\right)\]
  \item Final marking:
    \[
    \left(\bigwedge_{p \in M_f} m_{n, p}\right) \land
    \left(\bigwedge_{p \in P \setminus M_f} \lnot m_{n, p}\right)\]
  \item One and only one \(t_i\) for each \(i\):
    \[
    \bigwedge_{i = 1}^n \bigvee_{t \in T} (\tau_{i, t} \land
    \bigwedge_{t' \in T} \lnot \tau_{i, t'})\]
  \item The transitions are enabled when they fire:
    \[
    \bigwedge_{i = 1}^n \bigwedge_{t \in T}
    (\tau_{i, t} \implies \bigwedge_{p \in \preset{t}} m_{{i-1}, p})\]
  \item Token game (for safe Petri nets):
    \begin{eqnarray*}
      &&
      \bigwedge_{i = 1}^n \bigwedge_{t \in T}
      \bigwedge_{p \in \postset{t}}
      (\tau_{i, t} \implies m_{i, p})
      \\&&
      \bigwedge_{i = 1}^n \bigwedge_{t \in T}
      \bigwedge_{p \in \preset{t} \setminus \postset{t}}
      (\tau_{i, t} \implies \lnot m_{i, p})
      \\&&
      \bigwedge_{i = 1}^n \bigwedge_{t \in T}
      \bigwedge_{p \in P, p \not\in \preset{t}, p \not\in \postset{t}}
      (\tau_{i, t} \implies (m_{i, p} \iff m_{{i-1}, p}))
    \end{eqnarray*}
  \end{itemize}
  \item Now, the constraint that \(\gamma\) deviates from the observed traces
    (for every \(\sigma \in L\), \(\dist(\gamma, \sigma) \geq m\)) is coded as:
    \[\bigwedge_{\sigma \in L}\bigwedge_{k = 1}^m
    \bigvee_{i = 1}^n \delta_{i, \sigma, k}\]
    with the \(\delta_{i, \sigma, k}\) correctly affected
    w.r.t.\ \(\lambda(t_i)\) and \(\sigma_i\):
    \[\bigwedge_{\sigma \in L}\bigwedge_{k = 1}^m
    \bigwedge_{i = 1}^n \Big(\delta_{i, \sigma, k} \iff \bigvee_{t \in T,\
      \lambda(t) \ne \sigma_i} \tau_{i, t}\Big)\]
    and that for \(k \neq k'\), the \(k\)\textsuperscript{th} and
    \(k'\)\textsuperscript{th} mismatch correspond to different \(i\)'s (i.e.\ a
    given mismatch cannot serve twice):
    \[\bigwedge_{\sigma \in L} \bigwedge_{i = 1}^n
    \bigwedge_{k = 1}^{m-1} \bigwedge_{k' = k + 1}^m
    \lnot(\delta_{i, \sigma, k} \land \delta_{i, \sigma, k'})\]
\end{itemize}

\subsection{Size of the Formula}
In the end, the first part of the formula
(\(\gamma = \lambda(t_1) \dots \lambda(t_n) \in \Language(N)\)) is coded by a
Boolean formula of size \(O(n \times |T| \times |N|)\), with \(|N| \eqdef |T| +
|P|\).

The second part of the formula (for every \(\sigma \in L\), \(\dist(\gamma,
\sigma) \geq m\)) is coded by a Boolean formula of size
\(O(n \times m^2 \times |L| \times |T|)\).

The total size for the coding of the formula \(\Phi^n_m(N,L)\) is
\[O\left(n \times |T| \times \big(|N| + m^2 \times |L|\big)\right)\,.\]

\subsection{SAT-encoding of Anti-Alignments for Levenshtein's Edit Distance}
\label{sec:sat_edit_dist}

Our SAT-encoding of anti-alignments for Levenshtein's edit distance uses the
same boolean variables as the SAT-encoding of anti-alignments for Hamming
distance of the previous section, completed with $\delta$ variables used to
encode the edit distance.

Our encoding is based on the same relations that are used by the classical
dynamic programming recursive algorithm for computing the edit distance between
two words $u=\sequence{u_1, \dots, u_{n}}$ and $v=\sequence{v_1, \dots, v_{m}}$:
$$
 \left\{
 \begin{array}{ll}
   \dist(\sequence{u_1, \dots, u_{i}}, \epsilon) = i \\
   \dist(\epsilon,\sequence{ v_1, \dots, v_{j}}) = j \\
   \dist(\sequence{u_1, \dots, u_{i+1}}, \sequence{v_1, \dots, v_{j+1}}) = \\\qquad\left\{
   \begin{array}{l @{\quad\mbox{ if }} l}
     \dist(\sequence{u_1, \dots, u_i}, \sequence{v_1, \dots, v_j})
     & u_{i+1} = v_{j+1} \\
     \begin{array}{l@{}l}
       1 + \min(
       &\dist(\sequence{u_1, \dots, u_{i+1}}, \sequence{v_1, \dots, v_j}),\\
       &\dist(\sequence{u_1, \dots, u_i}, \sequence{v_1, \dots, v_{j+1}}))
     \end{array}
     & u_{i+1} \not = v_{j+1}
   \end{array}\right.
 \end{array}
 \right.
 $$



We encode this computation in a SAT formula \(\phi\) over variables
$\delta_{i,j,d}$, for \(i = 0, \dots, n\), \(j = 0, \dots, m\) and \(d = 0,
\dots, n+m\).
Formula \(\phi\) will have exactly one solution, in which each variable
$\delta_{i,j,d}$ is \(\texttt{true}\) iff \(u_1 \dots u_i\) and \(v_1\dots v_j\)
differ by at least $d$ editions.

In order to test equality between the \(u_{i}\) and \(v_{j}\), we use variables
$\lambda_{i,a}$ and $\lambda'_{j,a}$, for \(i = 0, \dots, n\), \(j =
0, \dots, m\) and \(a \in \Sigma\), and we set their value such that
\(\lambda_{i,a}\) is \(\texttt{true}\) iff \(u_i = a\), and \(\lambda'_{j,a}\) is
\(\texttt{true}\) iff \(v_j = a\). Hence, the test $u_{i+1} = v_{j+1}$ becomes in our
formulas: \(\bigvee_{a \in \Sigma} (\lambda_{i+1,a} \wedge \lambda'_{j+1,a})\).
For readability of the formulas, we refer to this coding by
\([u_{i+1} = v_{j+1}]\). We also write similarly \([u_{i+1} \neq v_{j+1}]\).

In the following, we describe the different clauses of the formula \(\phi\) of our SAT encoding of the edit distance.
\begin{eqnarray}
  &&\textstyle\delta_{0,0,0} \quad\land\quad \bigwedge_{ d > 0 } \neg\delta_{0,0,d}\\
  &&\textstyle\bigwedge_{ d  }  \bigwedge_{ i = 0 }^n \quad (\delta_{i+1,0,d+1} \Leftrightarrow \delta_{i,0,d})\\
  &&\textstyle\bigwedge_{ d  }  \bigwedge_{ j = 0 }^n \quad  (\delta_{0,j+1,d+1} \Leftrightarrow  \delta_{0,j,d})\\
  &&\textstyle\bigwedge\limits_{ d  }  \bigwedge\limits_{ i = 0 }^n \bigwedge\limits_{ j = 0 }^n \quad
  [u_{i+1} = v_{j+1}] \Rightarrow (\delta_{i+1,j+1,d} \Leftrightarrow \delta_{i,j,d})\\
  &&\textstyle\bigwedge\limits_{ d  }  \bigwedge\limits_{ i = 0 }^n \bigwedge\limits_{ j = 0 }^n \quad
  [u_{i+1} \neq v_{j+1}] \Rightarrow (\delta_{i+1,j+1,d+1} \Leftrightarrow (\delta_{i+1,j,d} \wedge \delta_{i,j+1,d}))\qquad
\end{eqnarray}



\begin{example}
At instants $i=1$ and $j=1$ of words $u=\sequence{s,g,c}$ and
$v=\sequence{s,b,c,a}$, the letters are the same, then, by (4), the distance is
only higher or equal to 0 : $(u_1=v_1) \Rightarrow (\delta_{1,1,0}
\Leftrightarrow \delta_{0,0,0})$.

However at instants $i=2$ and $j=2$, the letters $u_2$ and $v_2$ are different. A step before,  $\delta_{1,2,1}$ and $\delta_{2,1,1}$ are \texttt{true} because of the length of the subwords. Then, by (5), the distance at instants $i=2$ and $j=2$ is higher or equal to 2 : $\delta_{2,2,2}$. The result is understandable because the edit distance costs the deletion of $g$ and the addition of $b$ to transform $u$ to $v$.
\end{example}

In order to insert this encoding of Leventshein's edit distance into our
formulas for anti-alignments, we need to compute the edit distance between the
expected anti-alignment and every trace of the log, which requires to use variables
$\delta_{i,\sigma,j,k}$ for \(i = 0 \dots n\), \(\sigma \in L\), \(j = 0 \dots
|\sigma|\), \(k = 0, \dots, \max(n, |\sigma|)\) to represent the fact that
$u_1\dots u_i$ and $v_1\dots v_j$ differ by at least $k$ editions.

\subsection{Solving the Formula in Practice}
In practice, 
the coding of the formula
\(\Phi^n_m(N,L)\) can be done using the Boolean variables \(\tau_{i, t}\), \(m_{i, p}\)
and \(\delta_{i, \sigma, k}\).

Then we need to transform the formula in conjunctive normal form (CNF) in order
to pass it to the SAT solver. We use Tseytin's transformation
\cite{Tseytin} to get a formula in conjunctive normal form (CNF) whose size is
linear in the size of the original formula. The idea of this transformation is
to replace recursively the disjunctions \(\phi_1 \lor \dots \lor \phi_n\) (where
the \(\phi_i\) are not atoms) by the following equivalent formula:
\[
\exists x_1, \dots, x_n \quad \left\{
\begin{array}{ll}
  & x_1 \lor \dots \lor x_n\\\land
  & x_1 \implies \phi_1\\\land
  & \dots\\\land
  & x_n \implies \phi_n
\end{array}\right.\]
where \(x_1, \dots, x_n\) are fresh variables.

In the end, the SAT solver
tells us if there exists a run \(\gamma = \lambda(t_1) \dots \lambda(t_n) \in \Language(N)\)
which differs by at least \(m\) editions with every observed trace \(\sigma \in
L\). If a solution is found, we extract the run \(\gamma\) using the values
assigned by the SAT solver
to the Boolean variables \(\tau_{i, t}\).

\section{Using Anti-Alignments to Estimate Precision}
\label{sec:precision}

In this section we show how to use anti-alignments to estimate
precision of process models.
Remarkably, we show how to modify the definitions of~\cite{DBLP:conf/apn/ChatainC16,DBLP:conf/bpm/DongenCC16}
so that the new metric does not depend on a predefined length.
In Section~\ref{sec:axioms} we dive into the adherence of the metric
with respect to a recent proposal for properties of precision
metrics~\cite{TAX20181}.

\subsection{Precision}
\label{sec:log_based}
Our precision metric is an adaptation of our previous versions presented
in~\cite{DBLP:conf/apn/ChatainC16,DBLP:conf/bpm/DongenCC16}. It relies on
anti-alignments to find the model run that is as distant as possible to the log
traces. Like anti-alignments, the definition of precision is parameterized by a
distance \(\dist\). In the examples, we will specify each time if we use
Levenshtein's edit distance (Definition~\ref{def:Levenshtein}) or
Hamming distance (Definition~\ref{def:Hamming_dist}).

\begin{definition}[Precision]\label{def:Log-precision}
	Let $L$ be an event log and $N$ a model. We define precision as follows:
	$$P_{aa}(N,L) \eqdef 1 - \sup_{\gamma \in \Language(N)}\dist(\gamma,L)\,.$$
\end{definition}

For instance, consider the model and log shown in
Figure~\ref{fig:anti-alignment-ex}. With Levenshtein's distance, the full run
$\sequence{a,b,c,f,i,k}$ is a maximal anti-alignment. It is at distance
\(\frac{3}{13}\) to any of the log traces, and hence $P_{aa}(N,L) = 1 -
\frac{3}{13} = 0.77$.

\subsubsection{Handling Process Models with Loops}
\label{sec:loops}

Notice that a model with arbitrary long runs (i.e., a process model that contains loops) may cause the formula in Definition~\ref{def:Log-precision} to
converge to 0. This is a natural artifact of comparing a finite language (the event log), with a possibly infinite
language (the process model). Since process models in reality contain loops, an adaptation of the metric is done in this section, so that
it can also handle this type of models without penalizing severely the loops.

\begin{definition}[Precision for Models with Loops
  ]\label{def:Log-precision-loops}
  Let $L$ be an event log and $N$ a model. We define \(\epsilon\)-precision
  as follows:
  \[P_{aa}^{\epsilon}(N,L) \eqdef 1 - \sup_{\gamma \in \Language(N)}\frac{\dist(\gamma,
    L)}{(1+\epsilon)^{|\gamma|}}\]
  with some \(\epsilon \geq 0\) which is a parameter of this definition.
\end{definition}
Informally, the formula computes the anti-alignment that provides maximal distance with any trace in the log, and
at the same time tries to minimize its length. The penalization for the length is parametrized over the $\epsilon$\footnote{Although, admittedly, $\epsilon$
is a parameter that should be decided apriori, in practice one can use a particular value to this parameter thorough several instances, without impacting significantly the insights obtained through this metric.}.
Observe that \(P_{aa}^0(N,L)\) is precisely the precision \(P_{aa}(N,L)\) of Definition~\ref{def:Log-precision}.
By making Definition~\ref{def:Log-precision-loops} not dependant on a predefined length, it deviates from the log-based
precision metrics defined in previous work~\cite{DBLP:conf/apn/ChatainC16,DBLP:conf/bpm/DongenCC16}.

\begin{figure}[t]
\centering
\begin{minipage}[t]{0.5\linewidth}
\vspace{0pt}
\includegraphics[width=.95\textwidth]{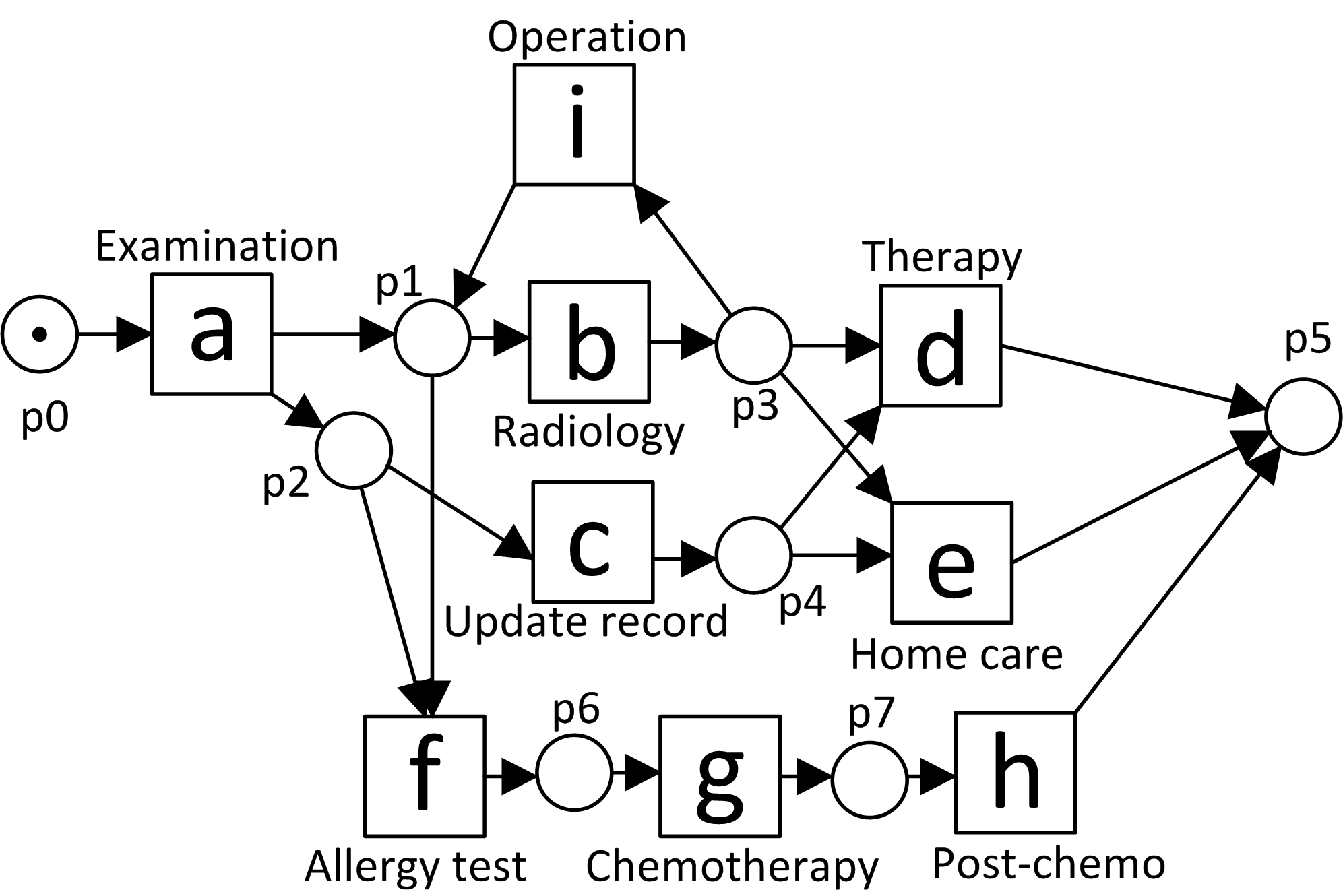}
\end{minipage}%
\caption{Example from~\cite{AdriansyahMCDA15}}
\label{fig:example_etcp}
\end{figure}

Let us now consider the model of Figure~\ref{fig:example_etcp},
and the log \mbox{$L=[\sequence{a},\sequence{a,b,c,d}, \sequence{a,f,g,h}, \sequence{a,b,i,b,c,d}]$}.
Assume that $\epsilon = 0.05$.
A possible anti-alignment is $\gamma_1 = \sequence{a,c,b,e}$ which is at least
at Levenshtein's distance $\frac12$ to any of the log traces.
For $\gamma_1$ the value of the formula is $1 - \frac{\frac12}{(1.05)^{4}} = 0.589$.
Another possible anti-alignment is $\gamma_2 = \sequence{a,c,b,i,b,i,b,i,b,i,b,e}$
which is at least at distance $\frac{10}{18}$ to any of the log traces.
For $\gamma_2$ the value of the formula is $1 -
\frac{\frac{10}{18}}{(1.05)^{12}} = 0.691$. Hence, since the anti-alignment that
maximizes the second term of the formula is $\gamma_1$, the precision computed
is $P_{aa}^{0.05}(N,L) = 0.589$. If instead, $\epsilon$ is set to a lower value,
e.g., $\epsilon = 0.02$, the corresponding value of the formula for
the anti-alignment $\sequence{a,c,b,i,b,i,b,i,b,i,b,i,b,i,b,i,b,e}$ will
be the mainimal, and therefore it will be selected as the anti-alignment
resulting in $P_{aa}^{0.02}(N,L) = 0.533$.

\subsubsection{Computing \(P_{aa}^{\epsilon}\)}

By incorporating the $\epsilon$ parameter in the definition of precision, now the metric can deal with models containing loops without predefining the length of the anti-alignment. In this section we show that the proposed extension is well-defined and can be computed, and provide some complexity results of the algorithms involved.

\begin{lemma}
  \label{lem:sup_epsilon_reached}
  For every finite model \(N\), log \(L\) and \(\epsilon > 0\), the
  supremum in the definition of \(P_{aa}^{\epsilon}\) is reached,
  i.e.\ there exists a full run \(\gamma \in \Language(N)\) such that
  \(P_{aa}^{\epsilon}(N,L) = 1 - \frac{\dist(\gamma, L)}{(1+\epsilon)^{|\gamma|}}\).
\end{lemma}
\begin{proof}
  Two cases have to be distinguished: if \(\Language(N) \subseteq L\), then the
  supremum equals \(0\), is obviously reached by any \(\gamma \in
  \Language(N)\), and \(P_{aa}^{\epsilon}(N,L) = 1\); otherwise,
  let \(\gamma_0 \in \Language(N) \setminus L\) and
  let \(m \eqdef \frac{\dist(\gamma_0, L)}{(1+\epsilon)^{|\gamma_0|}}\);
  we show that the supremum in the definition of \(P_{aa}^{\epsilon}\) becomes now a
  maximum over a finite set of runs, bounded by a given length \(n\) that
  depends on \(m\) and \(\epsilon\):
  \[
  \sup_{\gamma \in \Language(N)}
  \frac{\dist(\gamma, L)}{(1+\epsilon)^{|\gamma|}}
  =
  \max_{\gamma \in \Language(N), |\gamma|\leq n}
  \frac{\dist(\gamma, L)}{(1+\epsilon)^{|\gamma|}}
  \,,\]
  with \(n \eqdef \left\lfloor\frac{-\log m}{\log(1+\epsilon)}\right\rfloor\).
  Indeed, for every \(\gamma\) strictly longer than \(n\), we have
  \(\frac{\dist(\gamma, L)}{(1+\epsilon)^{|\gamma|}} <
  \frac1{(1+\epsilon)^n} = \exp(-n \cdot \log(1+\epsilon)) \leq
  \exp(\log m) = m\), which also shows that \(|\gamma_0| \leq n\). Hence
  \(\gamma_0\) is considered in our \(\max\), and then
  \(
  \max\limits_{\gamma \in \Language(N), |\gamma|\leq n}
  \frac{\dist(\gamma, L)}{(1+\epsilon)^{|\gamma|}}
  \geq m >
  \sup\limits_{\gamma \in \Language(N), |\gamma| > n}
  \frac{\dist(\gamma, L)}{(1+\epsilon)^{|\gamma|}}
  \).
  \qed
\end{proof}
Lemma~\ref{lem:sup_epsilon_reached} gives us the key for an algorithm to compute
\(P_{aa}^{\epsilon}\).

\begin{algorithm}
  \label{algo_precision}
  Algorithm for computing \(P_{aa}^{\epsilon}(N,L)\):
  \begin{itemize}
  \item if \(\Language(N) \not\subseteq L\),
    then
    \begin{itemize}
    \item select \(\gamma_0 \in \Language(N) \setminus L\)
    \item let \(m \eqdef \frac{\dist(\gamma_0, L)}{(1+\epsilon)^{|\gamma_0|}}\)
    \item explore the reachability graph of \(N\) until depth
      \(n \eqdef \left\lfloor\frac{-\log m}{\log(1+\epsilon)}\right\rfloor\) and return
      \(1 - \max_{\gamma \in \Language(N), |\gamma|\leq n}
      \frac{\dist(\gamma, L)}{(1+\epsilon)^{|\gamma|}}\);
    \end{itemize}
  \item else return \(1\) (the model has perfect precision).
  \end{itemize}
\end{algorithm}
The correctness of this algorithm follows directly from
Lemma~\ref{lem:sup_epsilon_reached}.
Its complexity resides essentially in the initial test, which corresponds to
simply deciding if \(P_{aa}^{\epsilon}(N,L) < 1\), whose complexity is given by
the following lemma:

\begin{lemma}
  \label{lem:complexity_precision_lt_1}
  The problem of deciding, for a finite model \(N\) and a log \(L\), if
  \(P_{aa}^{\epsilon}(N,L) < 1\), is equivalent to deciding reachability in Petri
  nets.
\end{lemma}
\begin{proof}
  We simply observe that \(P_{aa}^{\epsilon}(N,L) < 1\) iff \(\Language(N)
  \not\subseteq L\). Deciding this is equivalent to deciding
  reachability in Petri nets, as showed in Lemma~\ref{lemma-precision1}.
  \qed
\end{proof}

However, in practice, one would generally skip the first check and jump directly
to the exploration until some depth \(n\), possibly computed form a given
threshold \(m\), like the one given by the \(\gamma_0\) in
Algorithm~\ref{algo_precision}.
Notice that the algorithm explores less deep (i.e.\ \(n\) is smaller) when
\(m\) is large (close to 1), i.e.\ \(\gamma_0\) is close to the optimal
anti-alignment. We can summarize this with the following variation of
Algorithm~\ref{algo_precision}:

\begin{algorithm}
  \label{algo_precision_threshold}
  Algorithm for estimating \(P_{aa}^{\epsilon}(N,L)\) using a threshold \(0 < m
  \leq 1\) as input:
  \begin{itemize}
  \item explore the reachability graph of \(N\) until depth
    \(n \eqdef \left\lfloor\frac{-\log m}{\log(1+\epsilon)}\right\rfloor\)
  \item if the exploration finds a full run \(\gamma \in \Language(N) \setminus L\)
  \item then output ``\(P_{aa}^{\epsilon}(N,L) =
    1 - \max_{\gamma \in \Language(N), |\gamma|\leq n}
    \frac{\dist(\gamma, L)}{(1+\epsilon)^{|\gamma|}}\)''
  \item else output ``\(P_{aa}^{\epsilon}(N,L) \geq 1 - m\)''.
  \end{itemize}
\end{algorithm}

\begin{lemma}
  \label{lem:complexity_precision_threshold}
  For any fixed \(\epsilon > 0\), the problem of deciding, for a finite model
  \(N\), a log \(L\) and a rational constant \(m > 0\), if
  \(P_{aa}^{\epsilon}(N,L) < 1 - m\),
  is NP-complete.
\end{lemma}
\begin{proof}
  The proof is similar to the one of Lemma~\ref{lem:sup_epsilon_reached}; here,
  the bound \(m\) is given directly, and we have the same equality
  \[
  \sup_{\gamma \in \Language(N)}
  \frac{\dist(\gamma, L)}{(1+\epsilon)^{|\gamma|}}
  =
  \max_{\gamma \in \Language(N), |\gamma|\leq n}
  \frac{\dist(\gamma, L)}{(1+\epsilon)^{|\gamma|}}
  \,,\]
  with \(n \eqdef \left\lfloor\frac{-\log m}{\log(1+\epsilon)}\right\rfloor\).
  This means, in order to check that \(P_{aa}^{\epsilon}(N,L) < 1 - m\), it
  suffices to guess a full run \(\gamma\) of length \(|\gamma| \leq n\), where
  \(n\) depends linearly on the size of the representation of \(m\)
  (number of bits in the numerator and denominator).
  Then one can check in polynomial time that
  \(\frac{\dist(\gamma, L)}{(1+\epsilon)^{|\gamma|}} > m\).

  For completeness, we proceed like in
  Lemma~\ref{lemma-complexity_anti-alignment-bounded-length}: we reduce
  reachability of \(M_f\) in a 1-safe acyclic Petri net \(N\) to
  \(P_{aa}^{\epsilon}(N,L) < 1 - m\) with \(L = \emptyset\) and \(m =
  \frac1{(1+\epsilon)^{|T|+1}}\).
  \qed
\end{proof}

\subsection{Discussion about Reference Properties for Precision}
\label{sec:axioms}
Recently, an effort to consolidate a set of desired properties for precision metrics has been proposed~\cite{TAX20181}. Five axioms are described
that establish different features of a precision metric \(\precision(N, L)\). Summarizing, the axioms proposed in~\cite{TAX20181} are:

\begin{itemize}
 \item $A1$: A precision metric should be a function, i.e.\ it should be deterministic.
 \item $A2$: If a process model $N_2$ allows for more behavior not seen in a log $L$ than another model $N_1$ does, then $N_2$ should have a lower precision
   than $N_1$ regarding $L$:
   \[L \subseteq \Language(N_1) \subseteq \Language(N_2) \implies \precision(N_1, L) \geq \precision(N_2, L)\]
 \item $A3$: Let $N_1$ be a model that allows for the behavior seen in a log $L$, and at the same time its behavior is properly included in a model $N_2$
   whose language is $\Sigma^*$\footnote{Actually, \cite{TAX20181} writes
     ``\(\Language(N_1) \subset \mathcal{P}(\Sigma^*)\)'', with
     \(\mathcal{P}\) for powerset, but we believe this is a mistake.} (called a flower model).
   Then the precision of $N_1$ on $L$ should be strictly greater than the one for $N_2$.
 \item $A4$: The precision of a log on two language equivalent models should be equal:
   \[\Language(N_1) = \Language(N_2) \implies \precision(N_1, L) = \precision(N_2, L)\]
 \item $A5$: Adding fitting traces to a fitting log can only increase the precision of a given model with respect to the log:
   \[L_1 \subseteq L_2 \subseteq \Language(N) \implies \precision(N, L_1) \leq \precision(N, L_2)\]
\end{itemize}
In the aforementioned paper, it is shown that the previous version of our
antialignment-based precision metric (from~\cite{DBLP:conf/bpm/DongenCC16}) does
not satisfy axiom $A5$ (the satisfaction of the rest of axioms are declared as
{\em unknowns} in the paper). With the new version of the metric
presented in this paper, we here provide proofs for these axioms, except for
\(A3\). But at the same time, we show that any precision metric can be adapted
in order to satisfy \(A3\).

\begin{lemma}
\label{lem:A1}
The metric $P_{aa}^{\epsilon}$ (for any fixed \(\epsilon\)) 
satisfies $A1$.
\end{lemma}
\begin{proof}
  Everything in our definitions is functional.
\qed
\end{proof}

\begin{lemma}
\label{lem:A2}
The metric $P_{aa}^{\epsilon}$ 
satisfies $A2$.
\end{lemma}
\begin{proof}
Let $L \subseteq \Language(N_1) \subseteq \Language(N_2)$. The definitions
take the \(\sup_{\gamma \in \Language(N)}\) of an expression which does not
depend on the model. Since \(\Language(N_1) \subseteq \Language(N_2)\), the
\(\sup_{\gamma \in \Language(N_2)}\) ranges over a \(\subseteq\)-larger set than
the \(\sup_{\gamma \in \Language(N_1)}\). Therefore the result cannot be
smaller, and we get \(\precision(N_1, L) \geq \precision(N_2, L)\).
\qed
\end{proof}

Our metrics may not satisfy the strict inequality required by \(A3\): they
satisfy only a weaker version of \(A3\) with non-strict inequality, but, as
observed in \cite{TAX20181}, this is then simply subsumed by \(A2\). The authors
of \cite{TAX20181} precisely introduced \(A3\) after arguing that, in case of a
flower model, a strict inequality should be required.

Anyway, we show in Lemma~\ref{lem:A3} that any precision metric \(\precision\)
can be modified so that it satisfies this requirement of strict inequality
for the flower models.
\begin{lemma}
\label{lem:A3}
Let \(\precision\) be any precision metric. It is possible to define a metric
\(\precision'\) from \(\precision\) such that \(\precision'\) satisfies \(A3\):
it suffices to set the precision of the flower models to a value smaller than
all the other precision values (after possibly extending the target set of the
function). This guarantees \(A3\) and preserves all the other axioms.
\end{lemma}
\begin{proof}
  The new metric \(\precision'\) satisfies \(A3\) by construction. Moreover, if
  \(\precision\) is deterministic (\(A1\)), then \(\precision'\) also is. For
  preservation of \(A_2\), \(A_4\) and \(A_5\), it suffices to study the
  different cases (separate flower model and others) to show that the equality
  and non-strict inequalities are preserved.
\qed
\end{proof}
We consider that satisfying \(A3\) is a very artificial issue. However, if
really the transformation defined in Lemma~\ref{lem:A3} had to be implemented,
it would imply that, in order to compute the precision, one would have to decide
if the model \(N\) is a flower model, i.e.\ if \(\Language(N) = \Sigma^*\). This
is known as the universality problem.
This problem is, in theory, highly intractable\footnote{
  This universality problem is PSPACE-complete for non-deterministic finite
  state automata (NFSA) \cite{krotzsch:hal-01571398}, and here the NFSA to consider would
  be the reachability graph of \(N\) (for \(N\) \(k\)-bounded), which is
  exponential in the size of \(N\). Hence, deciding universality for
  \(k\)-bounded labeled Petri nets is in EXPSPACE.
}.
But again, this is very artificial: in practice it suffices to explore at a very
short finite horizon to detect many many non-flower models.

\begin{lemma}
\label{lem:A4}
The metric $P_{aa}^{\epsilon}$ 
satisfies $A4$.
\end{lemma}
\begin{proof}
This trivially holds since both metrics are behaviorally defined.
Also, we copied this axiom \(A4\) from \cite{TAX20181}, but observe that it is a
simple corollary of \(A2\) as soon as \(L \subseteq \Language(N_1) =
\Language(N_2)\).
\qed
\end{proof}

\begin{lemma}
\label{lem:A5}
Metrics $P_{aa}^{\epsilon}$ (for any \(\epsilon\)) satisfies $A5$.
\end{lemma}
\begin{proof}
  With \(L_1 \subseteq L_2\), for every \(\gamma \in \Language(N)\), we have
  \(\dist(\gamma, L_1) \geq \dist(\gamma, L_2)\), so the \(\sup_{\gamma \in
    \Language(N)}\) cannot be smaller for \(L_1\) than for \(L_2\). The rest
  does not depend on the log.
  \qed
\end{proof}




\section{Tool Support and Experiments}
\label{sec:experiments}

In this section we present the new tool implementing the results of this paper, and both a qualitative
and quantitative evaluation on state-of-the-art benchmarks from the literature. 
To compare the different distances based results, we denoted Leventshein distance based anti-alignment precision by $P_{aa}^{L}$ and Hamming distance based anti-alignment precision by $P_{aa}^{L}$.

\subsection{{\tt da4py}: A Python Library Supporting Anti-Alignments}
\label{sec:tool}

Several tools implement anti-alignments. \texttt{Darksider}, an Ocaml command line software, has already been presented in \cite{DBLP:conf/bpm/DongenCC16}. It creates the SAT formulas and calls the solver \texttt{Minisat+} \cite{minisat} to get the result. \texttt{ProM} software \cite{verbeek2010prom} also has an anti-alignment plugin, that computes anti-alignments in a brute force way. Recently, we have created a Python library in order to make our technique more accessible: \texttt{da4py} \footnote{https://github.com/BoltMaud/da4py, a Python version of \texttt{Darksider}}. Thanks to the use of the SAT library \texttt{PySAT} \cite{imms-sat18}, \texttt{da4py} allows one to run different state-of-the-art SAT solvers. Moreover, this SAT library uses an implementation of the RC2 algorithm \cite{ignatiev2018rc2} in order to get MaxSAT solutions, a variant that improves a lot the efficiency of computing anti-alignment. Finally,  \texttt{da4py} is compatible with the library \texttt{pm4py} \cite{berti2019process}, and uses the same data objects.   

Remarkably, in order to deal with large logs (as the ones shown in the quantitative evaluation part),  \texttt{da4py} has a variant that allows to compute a prefix of anti-alignments, thus alleviating the complexity by not requiring a full run but only a prefix. Accordingly, the corresponding precision measure is then a variant, that is normalized by the length of the anti-alignment prefix computed. Furthermore, for anti-alignments based on Levenshtein's distance, another simplification is to add a threshold on the number of editions (max\_d attribute) between the run and the traces, to compute a lower-bound for the anti-alignment instead of the complete anti-alignment. 

\subsection{Qualitative Comparison}
\label{sec:qualit}

\begin{table}
  \centering
  \footnotesize
  \begin{tabular}{l}
    Trace              \\
    \hline
    $\langle A,B,D,E,I \rangle$        \\
    $\langle A,C,D,G,H,F,I \rangle$   \\
    $\langle A,C,G,D,H,F,I \rangle$   \\
    $\langle A,C,H,D,F,I \rangle$      \\
    $\langle A,C,D,H,F,I \rangle$     \\
  \end{tabular}
  \medskip
  \caption{An example event log.}
  \label{tab:examplelog}
\end{table}

\begin{figure}[!p]
  \setlength{\intextsep}{15pt}
	\centering
	\begin{minipage}[c]{0.45\textwidth}
		\begin{figure}[H]
			\centering
			\includegraphics[width=\textwidth]{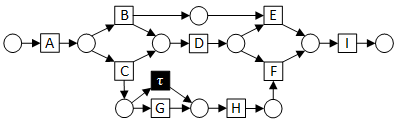}
			\caption{The ideal model. Fitting, fairly precise and properly generalizing.}
			\label{fig:model 1}
		\end{figure}
	\end{minipage}\hfill
	\begin{minipage}[c]{0.45\textwidth}
		\begin{figure}[H]
			\centering
			\includegraphics[width=\textwidth]{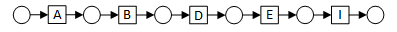}
			\caption{Most frequent trace. Precise, but not fitting or generalizing.}
			\label{fig:model 2}
		\end{figure}
	\end{minipage}

	\begin{minipage}[c]{0.3\textwidth}
		\begin{figure}[H]
			\centering
			\includegraphics[width=0.7\textwidth]{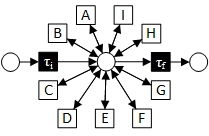}
			\caption{The flower model. Fitting and generalizing, but very imprecise.}
			\label{fig:model 3}
		\end{figure}
	\end{minipage}\hfill
	\begin{minipage}[c]{0.6\textwidth}
		\begin{figure}[H]
			\centering
			\includegraphics[width=\textwidth]{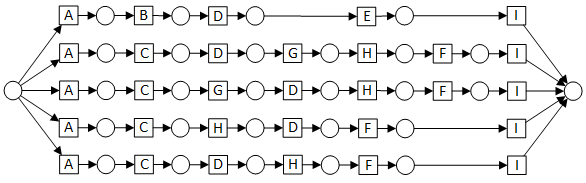}
			\caption{All traces separate. Fitting, precise, but not generalizing.}
			\label{fig:model 4}
		\end{figure}
	\end{minipage}

	\begin{minipage}[c]{0.45\textwidth}
		\begin{figure}[H]
			\centering
			\includegraphics[width=\textwidth]{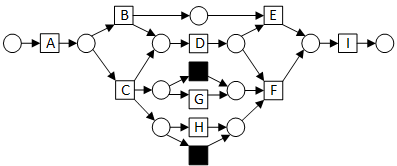}
			\caption{A model with G and H in parallel.}
			\label{fig:model 5}
		\end{figure}
	\end{minipage}\hfill
	\begin{minipage}[c]{0.45\textwidth}
		\begin{figure}[H]
			\centering
			\includegraphics[width=\textwidth]{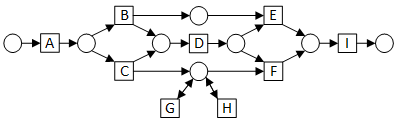}
			\caption{A model with G and H in self-loops}
			\label{fig:model 6}
		\end{figure}
	\end{minipage}
	\begin{minipage}[c]{0.45\textwidth}
		\begin{figure}[H]
			\centering
			\includegraphics[width=\textwidth]{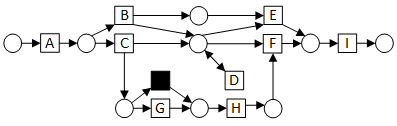}
			\caption{A model with D in a self-loop}
			\label{fig:model 7}
		\end{figure}
	\end{minipage}\hfill
	\begin{minipage}[c]{0.45\textwidth}
		\begin{figure}[H] 
			\centering
			\includegraphics[width=0.8\textwidth]{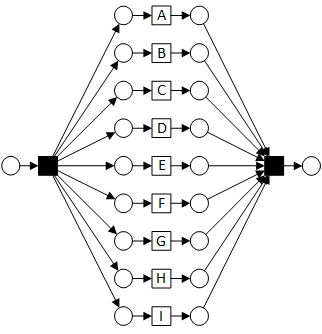}
			\caption{A model with all transitions in parallel.}
			\label{fig:model 11}
		\end{figure}
	\end{minipage}
	
\end{figure}

\begin{figure}[!t] 
	\begin{minipage}[c]{0.45\textwidth}
		\begin{figure}[H]
			\centering
			\includegraphics[width=0.9\textwidth]{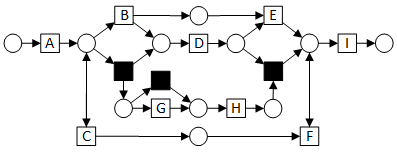}
			\caption{A model where C and F are in a loop, but need to be executed equally often to reach the final marking. 
				}
			\label{fig:model 9}
		\end{figure}
	\end{minipage}\hfill
\bigskip\bigskip
	\begin{minipage}[c]{0.45\textwidth} 
		\begin{figure}[H] 
			\centering
			\includegraphics[width=0.85\textwidth]{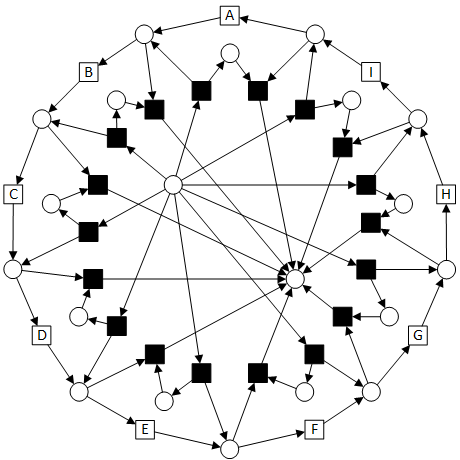}
			\caption{Round-robin model. The outer loop can be started at any point and then exited one transition before completing the loop.  
				}
			\label{fig:model 10}
		\end{figure}
	\end{minipage}
\end{figure}

A set of examples are taken from page 64 of \cite{rozinat2010conformance}, and consist of the simple event log shown in Table~\ref{tab:examplelog} aligned with 10 different process models. The log consists of only five different traces, with various frequencies. The models in Figures~\ref{fig:model 1} to \ref{fig:model 4} are four examples of models often used to show the differences between fitness, precision and generalization. The model in Figure~\ref{fig:model 1} shows the ``ideal'' process discovery result, i.e. the model that is fitting, fairly precise and properly generalizing.
The models in Figures~\ref{fig:model 5} to \ref{fig:model 11} present the same set of activities with varying loop and/or parallel constructs. Two new process models that describe particularly different routing logic from the previous models are depicted in Figures~\ref{fig:model 9} and \ref{fig:model 10}.




\begin{table}
  \centering
   \scriptsize
   \scalebox{0.9}{
  \begin{tabular}{l l || p{0.6cm} | p{0.6cm} || p{0.5cm} | p{0.6cm} || p{0.7cm}| p{0.7cm}|| p{0.6cm} | p{0.5cm} }
    Model		     &                  &$P_{ET}$&$P_{ETC}$&$P_{a}$&$P_{ne}$&$MAP^3$&$MAP^7$&$P^H_{aa}$ & $P^L_{aa}$\\ \hline
    Fig.~\ref{fig:model 1} &Generating model	&0.992	 &0.994	   &0.982  &0.995   & 0.880 &0.852     & 0.818&0.955 \\
    Fig.~\ref{fig:model 2} &Single trace	&1.000	 &1.000	   &1.000  &0.893   & 1.000&1.000    & 1.000& 1.000  \\
    Fig.~\ref{fig:model 3} &Flower model	&0.136	 &0.119	   &0.142  &0.117   &0.003  &0.000   &0.181&0.364 \\
    Fig.~\ref{fig:model 4} &Separate traces	&1.000	 &0.359	   &1.000  &0.985   &1.000&1.000   & 1.000&1.000  \\
    Fig.~\ref{fig:model 5} &G,H in parallel	&0.894	 &0.936	   &0.947  &0.950   &0.564&0.535 & 0.818&0.955\\
    Fig.~\ref{fig:model 6} &G,H as self-loops	&0.884	 &0.889	   &0.947  &0,874   &0.185&0.006    & 0.272&0.727\\
    Fig.~\ref{fig:model 7} &D as self-loop	&0.763	 &0.760	   &0.797  &0.720   &0.349&0.069&  0.272&0.727 \\
    Fig.~\ref{fig:model 11}&All parallel	&0.273	 &0.170	   &0.336  &0.158   &0.006&0.000&    0.181&0.500 \\
    Fig.~\ref{fig:model 9} &C,F equal loop	&0.820	 &0.589	   &0.839  &0.600  & $\quad$-- &$\quad$ --   &   0.818&0.910 \\
    Fig.~\ref{fig:model 10}&Round-robin	&0.579	 &0.185	   &0.889  &0.194   &0.496&0.274    & 0.181&0.636\\
  \end{tabular}}
  \smallskip
  \caption{Precision measures for all models}
  \label{tab:results}
\end{table}

\begin{table}
  \centering
   \scriptsize
   \scalebox{0.88}{
  \begin{tabular}{l l | c | c |}
    Model		     &                  & Hamming distance Anti-alignments&Edit distance Anti-alignments\\ \hline
    Fig.~\ref{fig:model 1} &Generating model	&$ \sequence{A, C, G, H, D, F, I}$ & $\sequence{A, C, G, H, D, F, I}$\\
    Fig.~\ref{fig:model 2} &Single trace	&$\sequence{A, B, D, E, I}$&$\sequence{A, B, D, E, I}$ \\
    Fig.~\ref{fig:model 3} &Flower model	&$\sequence{tau, D, I, C, C, E, C, F, D, I, tau}$&$\sequence{tau, G, G, G, G, G, G, G, G, G, tau}$\\
    Fig.~\ref{fig:model 4} &Separate traces	&$\sequence{A, C, G, D, H, F, I}$&  $\sequence{A, C, D, H, F, I}$\\
    Fig.~\ref{fig:model 5} &G,H in parallel	&$\sequence{A, C, D, H, G, F, I}$&$\sequence{A, C, tau, tau, D, F, I}$\\
    Fig.~\ref{fig:model 6} &G,H as self-loops	&$\sequence{A, C, H, H, D, G, G, G, G, F, I}$&$\sequence{A, C, G, G, G, G, G, D, G, F, I}$\\
    Fig.~\ref{fig:model 7} &D as self-loop	&$\sequence{A, C, D, D, D, D, G, H, D, F, I}$& $\sequence{A, B, D, D, D, D, D, D, D, E, I}$\\
    Fig.~\ref{fig:model 11}&All parallel	&$\sequence{tau, F, B, I, E, G, A, D, C, H, tau}$& $\sequence{tau, I, E, D, B, C, F, A, H, G, tau}$\\
    Fig.~\ref{fig:model 9} &C,F equal loop	&$\sequence{A, C, B, D, E, I}$&$\sequence{A, C, B, D, E, F, I}$\\
    Fig.~\ref{fig:model 10}&Round-robin	&$\sequence{tau, I, A, B, C, D, E, F, G, H, tau}$&$\sequence{tau, E, F, G, H, I, A, B, C, D, tau}$\\

  \end{tabular}}
  \smallskip
  \caption{Anti-alignments of maximal size 11 for all models : interestingly, in some cases the Hamming anti-alingments are equal or very similar to the Levenshtein ones, which validates the use of the former (whose SAT encoding is simpler) as alternative for the latter. Also notice that in some cases, the fact that we require in this qualitative comparison to deal with full runs implies very different anti-alignments are obtained through the two distances (e.g., Figure~\ref{fig:model 11}).}
  \label{tab:aa}
\end{table}

Table~\ref{tab:results} compares some precision metrics for the models in Figures~\ref{fig:model 1} to
\ref{fig:model 10} with the two possible metrics proposed in this paper: $P^H_{aa}$ and $P^L_{aa}$, representing the precision for the Hamming and Levenshtein distance, respectively. To understand the values provided by our proposed metrics, the reader
can find in Table~\ref{tab:aa} the anti-alignments computed for each process model, both for the Hamming and the Levenshtein distances.
Observe that, as a consequence of Lemma~\ref{lem:pH_lt_pL}, $P^H_{aa} \leq P^L_{aa}$ for all models.
The 
values of 
$P_{a}$ are defined as in~\cite{AryaThesis}. The precision values in $P_{ET}$ and 
$P_{ETC}$ are defined in \cite{AdriansyahMCDA15}.  The value $P_{ne}$ 
denotes the precision metric from~\cite{BrouckeWVB14}. Finally, the values $MAP^3$ and $MAP^7$
are defined in \cite{AugustoACDRR18}\footnote{Notice that the metrics $MAP^3$
  and $MAP^7$ are not applicable for one of the benchmarks of this paper, due to
  the existence of unbounded constructs.}; because they rely on behavioral
abstractions which are based on refinements of the directly-follows graph
\cite{leemans2013discovering}, they can dramatically overestimate the behavior
of the model, therefore they tend to be pessimistic. For instance, for well
chosen values of \(k\), \(MAP^k\) would use the same abstraction for two models
\(N_1\) and \(N_2\) having languages \(\Language(N_1) = a^*b^3c^* | d^*b^3e^*\)
and \(\Language(N_2) = (a^* | d^*) b^3 (c^* | e^*)\) and assign them the same
precision. Intuitively, this is not satisfactory: assume that the log contains
only traces in \(\Language(N_1)\), then model \(N_1\) should score much better than
\(N_2\).

Clearly, the existing 
precision metrics do not agree on all models and do not always agree with the intuition for precision. 
For example, the very precise model of Figure~\ref{fig:model 4} is considered to have a precision of $0.359$ by the 
$P_{ETC}$  metric. Also, the model of 
Figure~\ref{fig:model 6} scores very high in $P_{ET}$-$P_{ETC}$-$P_a$, although a trace with a thousand G's is 
possible in the model. Our metric $P^L_{aa}$ for this model however cannot fully penalize this, due to the requirement to deal
with full run anti-alignments (see the anti-alignment in Table~\ref{tab:aa}, which only differs with the log in the infix part). However, if we allowed for a larger run, e.g., doubling the length of the anti-alignment ($22$), the value would drop
from $P^L_{aa}= 0.727$ to $P^L_{aa}=0.613$ . 
A similar situation, with the same antidote, happens with the flower model, i.e., the value obtained for the Levenshtein based anti-alignment is higher than the rest, due to requiring the $\tau$ transitions that do not penalize in terms of distance.



%
%

\subsection{Quantitative Comparison}
\label{sec:quant}

In this section, we evaluate the techniques of this paper for 12 available real-life logs of varying sizes, and 
which have been recently used for benchmarking process discovery methods~\cite{AugustoCDRMMMS19}. They cover different fields, 
ranging from finance through to healthcare and government. The logs are publicly available at the 4DTU Data Center\footnote{\url{https://data.4tu.nl/repository/collection:event_logs_real}}. The experiments have been executed on a virtual machine with Debian 3.16.43-2+deb8u2 system, 12 CPU Intel Xeon 2.67GHz and 50GB RAM. 

\subsubsection{Benchmark Description}

For each of the twelve logs, we compute precision for the models automatically discovered with two state-of-the-art discovery algorithms: Inductive Miner \cite{leemans2013discovering}, and Split Miner \cite{augusto2017split}. We compare the performance of computing our metric against the one required for the $MAP$ metric \cite{AugustoACDRR18}, which is known to be the most recent precision contribution. 
Given that both methods are known to incur into a high complexity (specially in terms of memory footprint),
we use samples of size 10 and 100 of the initial log. Similarly to \cite{AugustoACDRR18}, we present average execution times of precision computations of the twelve models and the samples. This is enough to get an insight on the empirical positioning of our approach with respect to the aforementioned work.
Since the compared techniques strongly depend on models sizes, that vary between 8 to 150 transitions, we also report the minimal and maximum time performance of each method. 

The notation $P_{aa_{10}}$ in Tab.~\ref{tab:results_real} indicates that we use the simpler prefix variant of anti-alignments (see Section~\ref{sec:tool}). As explained in Section ~\ref{sec:tool}, a second parameter, $max\_d$, helps to reduce complexity of the SAT formula by bounding the number of maximal differences: in all but the last row, we use $max\_d =20$. In the last row we set $max\_d$ to $10$, due to the computational requirements of that last set of benchmarks.

In this section, we compare our work against $MAP$ with $k=3$ as advised by the authors. 


\subsubsection{Discussion}

By relying on approximations of anti-alignments (with prefixes instead of full runs, but also by limiting the maximal number of differences),  our approach tends to run in reasonable time. We compare our method to $MAP^3$ which has shown good performances in \cite{AugustoACDRR18}. For some models and logs, for instance model of BPI'2015 discovered with inductive miner and a log of size 10, we obtain better execution times for both Hamming and Levenshtein based precision. More generally Tab.~\ref{tab:results_real} shows that Hamming based precision, which gives an approximation of Levenshtein based precision, is often the quicker approach. We see that for the first time, our tool can deal reasonably with large problem instances, when it is compared to the previous implementations. It should be stressed that several engineering efforts can be incorporated into the current tool, to drastically reduce part of its computational demands.

\begin{table}
  \centering
  \begin{tabular}{| c| c| c| c| c| c| c| c|  }
   \hline
    \multirow{2}{*}{$|L|$} & \multirow{2}{*}{Method} &\multicolumn{3}{c|}{Inductive Miner} & \multicolumn{3}{c|}{Split Miner}\\ 
               \cline{3-8} 

    & & avg&min &max &avg &min &max \\  \hline
    \multirow{3}{*}{10 } & $MAP^3$ &102.314 &0.519 &896.299 &0.577 &0.439& 0.823\\ 
    & $P_{aa_{10}}^H$ &6.981 &0.903 &17.843& 6.154 &1.078& 14.707\\ 
    & $P_{aa_{10}}^L$ &121.95 &19.201 &246.876&102.103 &28.574 &189.876\\ \hline
         \multirow{3}{*}{100 } &$MAP^3$ &111.219 &0.612 &889.024 &0.759 &0.526 &1.129\\ 

         & $P_{aa_{10}}^H$ &16.92& 7.779 &30.492&15.544 &8.278& 27.059\\ 
    & $\sim P_{aa_{10}}^L$  &632.519 &99.269& 1229.862& 499.762 &138.501 &921.791\\ \hline
  \end{tabular}
  \smallskip
  \caption{Execution Times (in seconds) using the twelve real-life logs 
  obtained on a
          virtual machine with CPU Intel Xeon  2.67GHz and 50GB RAM.}
  \label{tab:results_real}
\end{table}

\section{Related Work}
\label{sec:related}


The seminal work in~\cite{RozinatA08} was the first one in relating observed behavior (in form of a set of
traces), and a process model. In order to asses how far can the model deviate from the log, the {\em follows} and {\em precedes} 
relations for both model and log are computed, storing for each relation whereas it {\em always} holds or only {\em sometimes}. In 
case of the former, it means that there is more variability. Then, log and model follows/precedes matrices are compared, and in those 
matrix cells where the model has a {\em sometimes} relation whilst the log has an {\em always} relation indicate that the model allows for 
more behavior, i.e., a lack of precision. This technique has important drawbacks: first, it is not general since in the presence of
loops in the model the characterization of the relations is not accurate~\cite{RozinatA08}. Second, the method requires a full 
state-space exploration of the model in order to compute the relations, a stringent limitation for models with large or even infinite 
state spaces.

In order to overcome the limitations of the aforementioned technique, a different approach was proposed in~\cite{JorgeMunozPhD}. The idea
is to find {\em escaping arcs}, denoting those situations where the model starts to deviate from the log behavior, i.e., events allowed
by the model not observed in the corresponding trace in the log. The exploration of escaping arcs is restricted by the log behavior, and 
hence the complexity of the method is always bounded. By counting how many escaping arcs a pair (model, log) has, one can estimate the 
precision of a model. Although being a sound estimation for the precision metric, it may hide the problems we are considering in this paper,
i.e., models containing escaping arcs that lead to a large behavior.

Less related is the work in~\cite{BrouckeWVB14}, where the introduction of {\em weighted artificial negative events} from a log is proposed. 
Given a log $L$, an artificial negative event is a trace $\sigma' = \sigma\cdot a$ where $\sigma \in L$, but $\sigma' \notin L$. Algorithms
are proposed to weight the confidence of an artificial negative event, and they can be used to estimate the precision and generalization
of a process model~\cite{BrouckeWVB14}. Like in~\cite{JorgeMunozPhD}, by only considering one step ahead of log/model's behavior, this 
technique may not catch serious precision/generalization problems.

Finally, an automata-based technique is presented in~\cite{Leemans20161}. The technique projects event log and process model onto all subsets
of activities of size $k$, and generates minimal deterministic finite automata that conjunctively represents both perspectives, so that precision
can be estimated by confronting the model and the conjunction automata.

As we mentioned in Section~\ref{sec:axioms}, a recent study elaborated on the guarantees provided by the aforementioned precision metrics~\cite{TAX20181}.
The study shows that none of the precision metrics up to that moment (including an earlier version of the metrics described in this paper~\cite{DBLP:conf/bpm/DongenCC16}) was able to satisfy all the properties together. Recently, however, a metric was proposed in~\cite{AugustoACDRR18}
which satisfies the 5 axioms. Likewise to the metric in~\cite{AugustoACDRR18}, the new precision metric proposed in this paper is shown to satisfy either these properties, or a weaker version of them.

\section{Conclusions and Future Work}
\label{sec:conclusions}

Conformance checking is becoming an important tool to certify the correct execution of processes in organizations. Quality
metrics for conformance checking are crucial to evaluate quantitatively process models, and among the four dimensions to
attain this task, precision is an important one. This paper proposes anti-alignments as a crucial artefact to compute an
accurate metric for precision. In contrast to existing metrics, anti-alignment based precision metrics satisfy most of the
properties for precision described in~\cite{TAX20181}.

As future work, we plan to explore new algorithms to compute anti-alignments so that the complexity of the problem can be
alleviated in practice. Also, we will consider new areas for the application of anti-alignments, like process model comparison or novel
ways of computing the generalization of process models (we already suggested one such approach in a recent paper~\cite{DBLP:conf/bpm/DongenCC16}).

We see potential for other notions of anti-alignments, for instance
anti-alignments can be generalized between two nets: to find behavior of the
former that cannot be found in the latter. The application of this generalization goes beyond the field of process mining: it can be
used, for instance, to allow generating behavior that is not possible
in a model, by simply aligning the net that allows for all the possible behaviors with the original net. This way, one can 
compute the most deviating trace of a net with respect to its own behavior, up to a given length. 
A second application of anti-alignments between two nets is for assessing behavioral process model 
similarity: this way, example
behavior that is in one model and not in the other can be provided as a hint on the dissimilarity between two models. Previous
techniques focus on the differences on either the causal relations, or those parts of the models
that are different~\cite{AbelArmasPhd,Dijkman08,PolyvyanyyWCRH14,WeidlichMW11,WeidlichPMW11}.

\subsubsection*{Acknowledgments.}
This work has been partially supported by funds from ENS Paris-Saclay (project
ProMAut of the Farman institute), the Spanish Ministry for Economy and
Competitiveness (MINECO) and the European Union (FEDER funds) under grant GRAMM
(TIN2017-86727-C2-1-R).

\bibliographystyle{plain}
\bibliography{Refs}
\end{document}